\documentclass[3p,10pt,a4paper,twoside,fleqn,sort&compress]{filomat}
\usepackage{amssymb,amsmath,latexsym,dsfont}
\usepackage[varg]{pxfonts}

\newtheorem{theorem}{Theorem}[section]

\newtheorem{definition}[theorem]{Definition}
\newtheorem{example}[theorem]{Example}

\newcommand{\FilVol}{00}
\newcommand{\FilYear}{0000}
\newcommand{\FilPages}{000--000}

\begin{document}

%

\title{Decision-making algorithm based on\\ the energy of interval-valued fuzzy soft sets}

\author[affil1]{Ljubica Djurovi\' c}
\ead{ljubica.milevic@pmf.kg.ac.rs}
\author[affil1]{Maja Lakovi\' c}
\ead{maja.lakovic@pmf.kg.ac.rs}
\author[affil1]{Nenad Stojanovi\' c}
\ead{nenad.s@kg.ac.rs}

\address[affil1]{Faculty of Science, University of Kragujevac, Radoja Domanovi\' ca 12, 34000 Kragujevac, Serbia}
\newcommand{\AuthorNames}{Lj. Djurovi\' c, M. Lakovi\' c, N. Stojanovi\' c }

\newcommand{\FilMSC}{Primary 03E72, 15A18 }
\newcommand{\FilKeywords}{Soft set, Interval-valued fuzzy sets, Interval-valued fuzzy soft sets, Energy, Singular values, Decision-making}
\newcommand{\FilCommunicated}{}
\newcommand{\FilSupport}{This work was supported by the Serbian Ministry of Science, Technological Development and Innovation (Agreement No. 451-03-65/2024-03/ 200122).}

\begin{abstract}
In our work, we continue to explore the properties of interval-valued fuzzy soft sets, which are obtained by combining interval-valued fuzzy sets and soft sets. We introduce the concept of energy of an interval-valued fuzzy soft set, as well as pessimistic and optimistic energy, enabling us to construct an effective decision-making algorithm. Through examples, the paper demonstrates how the introduced algorithm is successfully applied to problems involving uncertainty. Additionally, we compare the introduced method with other methods dealing with similar or related issues.
\end{abstract}

\maketitle

\makeatletter
\renewcommand\@makefnmark%
{\mbox{\textsuperscript{\normalfont\@thefnmark)}}}
\makeatother

\section{Introduction}
\label{}

The majority of problems in real-life situations, such as economics, engineering, environmental protection, social sciences, and medical sciences, often lack clear data. Therefore, due to various forms of uncertainty characterizing these problems, traditional methods are often not sufficiently successful. After Zadeh introduced fuzzy sets in 1965 (\cite{z1965}), many new theories dealing with imprecision and uncertainty have been developed. Some of these approaches involve expanding the theory of fuzzy sets, while others attempt to address imprecision and uncertainty in different ways. It is worth mentioning the theory of intuitionistic fuzzy sets introduced by Atanassov in 1986 (see \cite{atanasov}, \cite{atanasov2}), vague set theory proposed by Gau and Buehrer in 1993 (\cite{gau}), as well as rough set theory introduced by Pawlak in the early 1980s (see \cite{pawlak}, \cite{pawlak2}). Interval-valued fuzzy sets were independently introduced by Zadeh (\cite{z1975}), Grattan-Guiness (\cite{grattan}), Jahn (\cite{jahn}) and Sambuc (\cite{sambuc}) in the 1970s.  All these theories have their limitations when dealing with problems related to uncertainty and imprecision (see \cite{4}). We assume that the difficulties arise from the inadequacy of tools for parameterizing these theories. Molodtsov (\cite{4}) introduced soft sets in 1999 as a mathematical tool for solving problems of uncertainty and imprecision without the aforementioned difficulties. For this reason, the theory of soft sets began to be applied in various mathematical disciplines, such as operations research, game theory, measure theory, probability, Riemann integration, decision making, and machine learning, where it found significant use. The wide application of soft set theory in solving problems in real-life situations has accelerated its development, resulting in a large number of new research studies (see \cite{ali}, \cite{18}, \cite{19}, \cite{majumdar}, \cite{mukherjee}, \cite{pei}, \cite{sezgin}, \cite{yang}).

Maji et al. in \cite{lj1} also investigated fuzzy soft sets, where membership degrees are represented by real values between 0 and 1. These sets have found application in decision-making problems. However, fuzzy soft set models, as shown in the work \cite{lj1}, assume that membership degrees are crisp real values between 0 and 1. If we enable membership degrees in fuzzy soft sets to be represented by subintervals of $[0,1],$ it would provide greater flexibility and precision in modeling uncertainty. As a result, interval-valued fuzzy soft sets (IVFSS) were constructed in the work \cite{son} by combining interval-valued fuzzy sets and soft sets. Interval-valued fuzzy soft sets find the greatest application in decision making, as well as in data filling, information measure, etc. In the work \cite{yangint} a decision-making method based on evaluations for IVFSS was proposed, while in the work \cite{ma} an efficient decision-making approach involving additional objects was provided. Qin et al. proposed a new approach to decision-making using IVFSS based on the means of the contrast table in the work \cite{qin}.

Graph theory, as a mathematical discipline, represents a fundamental theory and a powerful tool in various fields, including computer science and chemistry. The concept of graph energy, as a numerical parameter, currently attracts the attention of many researchers. Ivan Gutman introduced the concept of graph energy in his work \cite{g1} in 1978 as the sum of the absolute values of the eigenvalues of a graph. The study of matrices and their properties, especially eigenvalues and singular values, forms the basis for analyzing graph energy. The introduction of graph energy has brought about many new scientific results (see \cite{d1}, \cite{g2}, \cite{g3}, \cite{j1}, \cite{n1}, \cite{n2}, \cite{z1}). Additionally, in the work \cite{mudric} energy and $\lambda$-energy of fuzzy soft sets were introduced, which have found significant application in decision making. The concept of energy appearing in these studies encourages the idea of introducing a similar concept, namely, the energy of interval-valued fuzzy soft sets as a numerical characteristic, along with possible application in the decision-making process.

In Section 2 of our paper, we present the fundamental concepts of soft set theory, fuzzy soft set theory, as well as interval-valued fuzzy soft set theory. In Section 3, we introduce new concepts, such as pessimistic and optimistic energy of interval-valued fuzzy soft sets, as well as their energy. These energies play a significant role in creating decision-making algorithms. In Section 4, we provide basic properties and boundaries for the introduced energies, while Section 5 is dedicated to the application of defined energies in decision-making. We will provide a decision-making algorithm based on these energies, and illustrate the proposed algorithm using two examples, comparing the obtained results with those obtained using other methods in the works \cite{ma}, \cite{qin} i \cite{yangint}. In Section 6, we summarize the obtained results from which certain conclusions can be drawn. Additionally, guidelines for further investigation of the introduced energies and their potential applications are provided.

\section{Preliminaries}
\label{}

In this section, we will provide basic definitions of soft sets, fuzzy sets, and fuzzy soft sets. Then we will introduce the basic concepts of interval-valued fuzzy sets, as well as interval-valued fuzzy soft sets. More detailed properties of these sets, along with numerous examples, can be found in \cite{ali}, \cite{cag}, \cite{11}, \cite{19}, \cite{lj1}, \cite{majumdar}, \cite{pei}, \cite{prade}, \cite{sezgin}, \cite{son}, \cite{yangint}, \cite{z1965} i \cite{z1975}.

Let $U$ be the universal set, $E$ be the set of parameters and $A\subseteq E.$ Using  $P(U),$ we denote the power set of the universe $U$.

\begin{definition}\cite{cag}
A soft set $F_A$ over the universe $U$ is a set determined by the mapping $f_A: E\rightarrow P(U),$ where $f_A(x)=\emptyset$ whenever $x\notin A.$
\end{definition}

The mapping $f_A$  is commonly referred to as the approximating function of the soft set $F_A$ for each $x\in E$. The soft skup $F_A$ over $U$ can also be represented using ordered pairs in the following way:
\[
F_A=\left\{\left(x,f_A(x)\right) \big|  x \in E, f_A(x) \in P(U)\right\}.
\]

We will denote the collection of all soft sets over the universe $U$ as $S(U)$.

\begin{definition}\cite{z1965}
A fuzzy set $X$ over the universe $U$ is a set determined by the mapping $\mu_X:U \rightarrow [0,1]$.
\end{definition}

The mapping $\mu_X$ is called the membership function of $X,$ while the value $\mu_X(u)$ is called the degree of membership of element $u \in U$ in the fuzzy set $X$.

A fuzzy set $X$ over $U$ can be represented in the following way:
\[
X=\left\{\left(\mu_X(u)/u\right) \big| u \in U, \mu_X(u) \in [0,1]\right\}.
\]

We will denote the collection of all fuzzy sets over the universe $U$ as $F(U)$.

By introducing the concept of fuzzy sets into the theory of soft sets, we obtain the concept of fuzzy soft sets.

\begin{definition}\cite{11}
A fuzzy soft set $\Gamma_A$ over the universe $U$ is a set determined by the mapping $\gamma_A : E \rightarrow F(U),$ where $\gamma_A(x)=\emptyset$ whenever $x \notin A.$
\end{definition}

Then $\gamma_A$ is called the fuzzy approximating function of the fuzzy soft set $\Gamma_A,$ while the value $\gamma_A(x)$ is the set called the $x$-element of the fuzzy soft set for all $x \in E$. Hence, the fuzzy soft set $\Gamma_A$ over the universe $U$ can be, similarly as soft sets, represented by a set of ordered pairs in the following way:
\[
\Gamma_A=\left\{\left(x,\gamma_A(x)\right)\big | x \in E, \gamma_A(x) \in F(U)\right\}.
\]

We will denote the collection of all fuzzy soft sets over the universe $U$ as $FS(U)$.

\begin{definition}\cite{z1975}
An interval-valued fuzzy set $\hat{X}$ over the universe $U$ is a set defined by mapping\linebreak $i_{\hat{X}}: U\rightarrow Int\left([0,1]\right),$ where $Int\left([0,1]\right)$  is the set of all closed subintervals of $[0,1].$
\end{definition}

The mapping $i_{\hat{X}}$ is called the interval-membership function of $\hat{X}$. If for every $u\in U$ we denote, respectively, with $i^-_{\hat{X}}(u)$ and $i^+_{\hat{X}}(u)$ the lower and upper degrees of membership $u$ to $\hat{X},$ where\linebreak $0\leqslant i^-_{\hat{X}}(u)\leqslant i^+_{\hat{X}}(u)\leqslant1,$ then we can denote  $i_{\hat{X}}(u)$ as
\[
i_{\hat{X}}(u)=\left[i^-_{\hat{X}}(u),i^+_{\hat{X}}(u)\right]
\]
and the value $i_{\hat{X}}(u)$ is called the degree of membership an element $u$ to $\hat{X}.$

We can also represent the interval-valued fuzzy set $\hat{X}$ in the following way:
\[
\hat{X}=\left\{\left(i_{\hat{X}}(u)/u\right)\big | u\in U, i_{\hat{X}}(u)\in Int\left([0,1]\right)\right\}.
\]

We will denote the collection of all interval-valued fuzzy sets over the universe $U$ as $IVFS(U)$.

\begin{definition}\cite{son}
  An interval-valued fuzzy soft set (IVFSS) $\mathcal{F}_A$ over the universe $U$  is a set defined by the mapping $\eta_A: E\rightarrow IVFS(U),$ where $\eta_A(x)=\emptyset$ whenever $x\notin A.$
\end{definition}

Similar to soft sets and fuzzy soft sets, the mapping $\eta_A$ is called the interval-valued fuzzy approximating function of the interval-valued fuzzy soft set $\mathcal{F}_A,$ and the value $\eta_A(x)$ is a set called an interval-valued fuzzy value set of the parameter $x\in E.$ Then $\mathcal{F}_A$ can be represented in the following way:
\[
\mathcal{F}_A=\left\{\left(x,\eta_A(x)\right)\big | x\in E, \eta_A(x)\in IVFS(U)\right\}.
\]

We will denote the collection of all interval-valued fuzzy soft sets over the universe $U$ as $IVFSS(U)$. In the following, we will use the notations $\mathcal{F}_A,\mathcal{F}_B,\mathcal{F}_C,\ldots$ for interval-valued fuzzy soft sets and $\eta_A,\eta_B,\eta_C,\ldots$ for their interval-valued fuzzy approximating functions, respectively.

We will provide an example of an interval-valued fuzzy soft set from the work \cite{yangint}, which we will use to illustrate the main concepts in our work.

\begin{example}\label{primer1}\cite{yangint}
  Suppose that $U=\{u_1,u_2,u_3,u_4,u_5,u_6\}$ is the set of the houses under consideration and\linebreak $A=\{x_1,x_2,x_3,x_4\}$ is the set of parameters, where for $i=1,2,3,4,$ the parameters $x_i$ represent, in order, \linebreak ''beautiful'', ''wooden'', ''cheap'' and ''in the green surroundings''. We define an interval valued fuzzy soft set $\mathcal{F}_A$ as follows:
  \begin{align*}
    \eta_A(x_1) &= \left\{([0.7,0.9]/u_1),([0.6,0.8]/u_2),([0.5,0.6]/u_3),([0.6,0.8]/u_4),([0.8,0.9]/u_5),([0.8,1.0]/u_6)\right\}, \\
    \eta_A(x_2) &= \left\{([0.6,0.7]/u_1),([0.8,1.0]/u_2),([0.2,0.4]/u_3),([0.0,0.1]/u_4),([0.1,0.3]/u_5),([0.7,0.8]/u_6)\right\}, \\
    \eta_A(x_3) &= \left\{([0.3,0.5]/u_1),([0.8,0.9]/u_2),([0.5,0.7]/u_3),([0.7,1.0]/u_4),([0.9,1.0]/u_5),([0.2,0.5]/u_6)\right\}, \\
    \eta_A(x_4) &= \left\{([0.5,0.8]/u_1),([0.9,1.0]/u_2),([0.7,0.9]/u_3),([0.6,0.8]/u_4),([0.2,0.5]/u_5),([0.7,1.0]/u_6)\right\}.
  \end{align*}
\end{example}

Two interval-valued fuzzy soft sets can be compared in the following way.

\begin{definition}\cite{son}
Let $A, B \subseteq E$ and let $\mathcal{F}_A$ and $\mathcal{F}_B$ be two interval-valued fuzzy soft sets over the universe $U$. We say that $\mathcal{F}_A$  is an interval-valued fuzzy soft subset of $\mathcal{F}_B$ if the following holds:

(1) $A\subseteq B$,

(2) $\forall x\in A$, $\eta_{A}(x)$  is an interval-valued fuzzy subset of $\eta_{B}(x),$
which we denote as $\mathcal{F}_A \widetilde{\subseteq}\mathcal{F}_B$.
\end{definition}

We say that $\mathcal{F}_A$ is an interval-valued fuzzy soft super set of $\mathcal{F}_B$ if $\mathcal{F}_B$ is an interval-valued fuzzy soft subset of $\mathcal{F}_A,$ which we denote as $\mathcal{F}_A\widetilde{\supseteq}\mathcal{F}_B.$

\begin{definition}\cite{son}
Two interval-valued fuzzy soft sets $\mathcal{F}_A$ and $\mathcal{F}_B$ are interval-valued fuzzy soft equal if $\mathcal{F}_A$ is an interval-valued fuzzy soft subset of $\mathcal{F}_B$ and $\mathcal{F}_B$ is an interval-valued fuzzy soft subset of $\mathcal{F}_A,$ which we denote as $\mathcal{F}_A=\mathcal{F}_B.$
\end{definition}

The operations complement, union, intersection, AND, and OR, as well as their properties, can be found in \cite{son} and \cite{yangint}.

\section{Energy of an interval-valued fuzzy soft set}
\label{}

Within this section, we introduce the concept of energy of an interval-valued fuzzy soft set, representing a numerical measure that encapsulates the interval-valued fuzzy soft set's attributes. These numerical values serve as parameters for scrutinizing conclusions in decision-making contexts, as further explored in Section 5.

Let's start with the tabular representation of the interval-valued fuzzy soft set.

Let $U=\{u_1,u_2,...,u_n\}$ i $A=\{x_1,x_2,...,x_m\}\subseteq E$. Then an interval-valued fuzzy soft set $\mathcal{F}_A$ can be represented using the following table:
\[
\begin{array}{r|cccc}
  \mathcal{F}_A & x_1 & x_2 & \cdots & x_m\\ \hline
   u_1 & i_{\eta_{A}(x_{1})}(u_1) & i_{\eta_{A}(x_{2})}(u_1) & \cdots & i_{\eta_{A}(x_{m})}(u_1) \\
   u_2 & i_{\eta_{A}(x_{1})}(u_2) & i_{\eta_{A}(x_{2})}(u_2) & \cdots & i_{\eta_{A}(x_{m})}(u_2) \\
       \vdots & \vdots &\vdots & \ddots & \vdots \\
   u_n & i_{\eta_{A}(x_{1})}(u_n) & i_{\eta_{A}(x_{m})}(u_n) &\cdots & i_{\eta_{A}(x_{m})}(u_n)\\
\end{array}.
\]

Using that $i_{\eta_A(x_j)}(u_i)=\left[i^-_{\eta_A(x_j)}(u_i),i^+_{\eta_A)x_j)}(u_i)\right],$ for $i=1,\ldots,n$ and $j=1,\ldots,m,$ the interval-valued fuzzy soft set $\mathcal{F}_A$ will be uniquely determined by the following two matrices.

\begin{definition}
Let $U=\{u_1,u_2,...,u_n\}$ and $A=\{x_1,x_2,...,x_m\}\subseteq E$. The matrix of minimum values of the interval-valued fuzzy soft set $\mathcal{F}_A$ is a matrix $\Psi_{\mathcal{F}_A}^{\min}$ of type $n\times m$ given by
  \[
 \Psi_{\mathcal{F}_A}^{\min}= \begin{bmatrix}
    i^-_{\eta_{A}(x_{1})}(u_1) & i^-_{\eta_{A}(x_{2})}(u_1) & \cdots & i^-_{\eta_{A}(x_{m})}(u_1) \\
    i^-_{\eta_{A}(x_{1})}(u_2) & i^-_{\eta_{A}(x_{2})}(u_2) & \cdots & i^-_{\eta_{A}(x_{m})}(u_2) \\
    \vdots &\vdots & \ddots & \vdots \\
    i^-_{\eta_{A}(x_{1})}(u_n) & i^-_{\eta_{A}(x_{m})}(u_n) &\cdots & i^-_{\eta_{A}(x_{m})}(u_n)
  \end{bmatrix}.
  \]
\end{definition}

Similarly, we define the matrix of maximum values of the interval-valued fuzzy soft set $\mathcal{F}_A.$
\begin{definition}
Let $U=\{u_1,u_2,...,u_n\}$ and $A=\{x_1,x_2,...,x_m\}\subseteq E$. The matrix of maximum values of the interval-valued fuzzy soft set $\mathcal{F}_A$ s a matrix $\Psi_{\mathcal{F}_A}^{\max}$ of type $n\times m$ given by
  \[
 \Psi_{\mathcal{F}_A}^{\max}= \begin{bmatrix}
    i^+_{\eta_{A}(x_{1})}(u_1) & i^+_{\eta_{A}(x_{2})}(u_1) & \cdots & i^+_{\eta_{A}(x_{m})}(u_1) \\
    i^+_{\eta_{A}(x_{1})}(u_2) & i^+_{\eta_{A}(x_{2})}(u_2) & \cdots & i^+_{\eta_{A}(x_{m})}(u_2) \\
    \vdots &\vdots & \ddots & \vdots \\
    i^+_{\eta_{A}(x_{1})}(u_n) & i^+_{\eta_{A}(x_{m})}(u_n) &\cdots & i^+_{\eta_{A}(x_{m})}(u_n)
  \end{bmatrix}.
  \]
\end{definition}

Let's consider Example \ref{primer1} listed in Section 2 of this paper.

\begin{example}\label{primer11}
The interval-valued fuzzy soft set $\mathcal{F}_A$ from Example \ref{primer1} can be represented using the matrix of minimum values and the matrix of maximum values as follows:
\[
\Psi_{\mathcal{F}_A}^{\min}= \begin{bmatrix}
  0.7 & 0.6 & 0.3 & 0.5 \\
  0.6 & 0.8 & 0.8 & 0.9 \\
  0.5 & 0.2 & 0.5 & 0.7 \\
  0.6 & 0.0 & 0.7 & 0.6 \\
  0.8 & 0.1 & 0.9 & 0.2 \\
  0.8 & 0.7 & 0.2 & 0.7
\end{bmatrix},\qquad \Psi_{\mathcal{F}_A}^{\max}=\begin{bmatrix}
  0.9 & 0.7 & 0.5 & 0.8 \\
  0.8 & 1.0 & 0.9 & 1.0 \\
  0.6 & 0.4 & 0.7 & 0.9 \\
  0.8 & 0.1 & 1.0 & 0.8 \\
  0.9 & 0.3 & 1.0 & 0.5 \\
  1.0 & 0.8 & 0.5 & 1.0
\end{bmatrix}.
\]
\end{example}

Now that we have defined the matrix of minimum values and the matrix of maximum values, we can also define the pessimistic and optimistic energies of the interval-valued fuzzy soft set.

\begin{definition}
  The pessimistic energy of the interval-valued fuzzy soft set $\mathcal{F}_A$, denoted as $\mathbb{E}^{\min}_{\mathcal{F}_A},$ is defined as
  \[
  \mathbb{E}^{\min}_{\mathcal{F}_A}=\sum\limits_{i=1}^{n}\sigma_i,
  \]
  where $\sigma_1\geqslant\sigma_2\geqslant\cdots\geqslant\sigma_n\geqslant0$ are the singular values of the matrix $\Psi^{\min}_{\mathcal{F}_A}$ of minimum values of the interval-valued fuzzy soft set $\mathcal{F}_A.$
\end{definition}

Let's return to Examples \ref{primer1} and \ref{primer11}.

\begin{example}
  The singular values of the matrix $\Psi^{\min}_{\mathcal{F}_A}$ from Example \ref{primer11} can easily be determined by finding the matrix $\Psi^{\min}_{\mathcal{F}_A}\cdot\left(\Psi^{\min}_{\mathcal{F}_A}\right)^T$ and computing its eigenvalues. Then, the singular values of the matrix $\Psi^{\min}_{\mathcal{F}_A}$ are the square roots of the eigenvalues of $\Psi^{\min}_{\mathcal{F}_A}\cdot\left(\Psi^{\min}_{\mathcal{F}_A}\right)^T:$
  \[
  \sigma_1=2.813267,\quad \sigma_2=0.877801,\quad \sigma_3=0.475986,\quad \sigma_4=0.358376,\quad \sigma_5=\sigma_6=0,
  \]
so based on the definition mentioned above,
  \[
  \mathbb{E}^{\min}_{\mathcal{F}_A}=\sum\limits_{i=1}^{6}\sigma_i=2.813267+0.877801+0.475986+0.358376+0+0=4.52543.
  \]
\end{example}

Similarly, we define the optimistic energy of the interval-valued fuzzy soft set $\mathcal{F}_A.$

\begin{definition}
   The optimistic energy of the interval-valued fuzzy soft set $\mathcal{F}_A$, denoted as $\mathbb{E}^{\max}_{\mathcal{F}_A},$ is defined as
  \[
  \mathbb{E}^{\max}_{\mathcal{F}_A}=\sum\limits_{i=1}^{n}\sigma_i,
  \]
  where $\sigma_1\geqslant\sigma_2\geqslant\cdots\geqslant\sigma_n\geqslant0$ are the singular values of the matrix $\Psi^{\max}_{\mathcal{F}_A}$ of maximum values of the interval-valued fuzzy soft set $\mathcal{F}_A.$
\end{definition}

Let's once again consider Examples \ref{primer1} and \ref{primer11}.

\begin{example}
   The singular values of the matrix $\Psi^{\max}_{\mathcal{F}_A}$ from Example \ref{primer11} are
   \[
   \sigma_1=3.721438,\quad \sigma_2=0.828772,\quad \sigma_3=0.364174,\quad \sigma_4=0.348435,\quad \sigma_5=\sigma_6=0,
   \]
so based on the previous definition,
   \[
   \mathbb{E}^{\max}_{\mathcal{F}_A}=\sum\limits_{i=1}^{6}\sigma_i=3.721438+0.828772+0.364174+0.348435+0+0=5.262819.
   \]
\end{example}

Now, to address the question of uniqueness of energy for two different interval-valued fuzzy soft sets, we will introduce the concept of energy of an interval-valued fuzzy soft set.

\begin{definition}
  The energy of the interval-valued fuzzy soft set  $\mathcal{F}_A,$ denoted as $\mathbb{E}^*_{\mathcal{F}_A}$, is defined as
  \[
  \mathbb{E}^*_{\mathcal{F}_A}=\frac{\mathbb{E}^{\min}_{\mathcal{F}_A}+\mathbb{E}^{\max}_{\mathcal{F}_A}}{2}.
  \]
\end{definition}

The energy of the interval-valued fuzzy soft set $\mathcal{F}_A$ from Example \ref{primer1} is $\mathbb{E}^*_{\mathcal{F}_A}=4.8941245.$

In Section 4, we present the basic properties of the defined energies, while in Section 5, we discuss the potential applications of the obtained energies in decision-making.

\section{Properties of energy of an interval-valued fuzzy soft set}
\label{}
In this section, we will provide upper bounds for the introduced energies of IVFSS.

\begin{theorem}
  Let $\mathcal{F}_A$ be an interval-valued fuzzy soft set, $U=\{u_1,u_2,\ldots,u_n\},$ $E=\{x_1,x_2,\ldots,x_m\}$ and $A\subseteq E.$ Let $\sigma_1,\sigma_2,\ldots,\sigma_n$ be the singular values of the matrix $\Psi_{\mathcal{F}_A}^{\min}$ representing the matrix of minimal values of the interval-valued fuzzy soft set $\mathcal{F}_A$. Then for the pessimistic energy of $\mathcal{F}_A$ holds:
$$\mathbb{E}^{\min}_{\mathcal{F}_A} \leqslant n\sqrt{m}.$$
\end{theorem}

\begin{proof}
  Using the inequality between the arithmetic mean and the quadratic mean applied to the singular values $\sigma_1,\sigma_2,\ldots,\sigma_n,$ we obtain
$$\mathbb{E}^{\min}_{\mathcal{F}_A} = \sum_{i=1}^{n}\sigma_{i}\leqslant\sqrt{n\sum_{i=1}^{n}\sigma_{i}^2}.$$

Knowing the basic properties of matrices, eigenvalues, and singular values, as well as the fact that $i^-_{\eta_A(x_j)}(u_i)\leqslant1,$ for $i=1,\ldots,n$ and $j=1,\ldots,m,$ we get
$$\sum_{i=1}^{n}\sigma_{i}^2={\mathrm{tr}}\left(\Psi_{\mathcal{F}_A}^{\min}\cdot\left(\Psi_{\mathcal{F}_A}^{\min}\right)^T\right)=\sum\limits_{i=1}^{n}\sum\limits_{j=1}^{m}\left(i^-_{\eta_A(x_j)}(u_i)\right)^2\leqslant mn.$$

Therefore, $\mathbb{E}^{\min}_{\mathcal{F}_A} \leqslant n\sqrt{m}.$
\end{proof}

The same applies to the optimistic energy of $\mathcal{F}_A.$

\begin{theorem}
  Let $\mathcal{F}_A$ be an interval-valued fuzzy soft set, $U=\{u_1,u_2,\ldots,u_n\},$ $E=\{x_1,x_2,\ldots,x_m\}$ and $A\subseteq E.$ Let $\sigma_1,\sigma_2,\ldots,\sigma_n$ be the singular values of the matrix $\Psi_{\mathcal{F}_A}^{\max}$ representing the matrix of maximum values of the interval-valued fuzzy soft set $\mathcal{F}_A$. Then for the optimistic energy of $\mathcal{F}_A$ holds:
$$\mathbb{E}^{\max}_{\mathcal{F}_A} \leqslant n\sqrt{m}.$$
\end{theorem}

Using the previous two theorems, we obtain the upper bound of the energy of an interval-valued fuzzy soft set.

\begin{theorem}
  Let $\mathcal{F}_A$ be an interval-valued fuzzy soft set, $U=\{u_1,u_2,\ldots,u_n\},$ $E=\{x_1,x_2,\ldots,x_m\}$ and $A\subseteq E.$ Then for the energy of $\mathcal{F}_A$ holds:
$$\mathbb{E}^{*}_{\mathcal{F}_A} \leqslant n\sqrt{m}.$$
\end{theorem}

\begin{proof}
  Since matrices $\Psi_{\mathcal{F}_A}^{\min}$ and $\Psi_{\mathcal{F}_A}^{\max}$ are of the same type $n\times m,$ they both have the same number of singular values, so using the previous theorems and a definition of an energy of $\mathcal{F}_A$, we get
  
  $\mathbb{E}^{*}_{\mathcal{F}_A} =\frac{\mathbb{E}^{\min}_{\mathcal{F}_A} +\mathbb{E}^{\max}_{\mathcal{F}_A} }{2}\leqslant\frac{n\sqrt{m}+n\sqrt{m}}{2}=n\sqrt{m}.$
\end{proof}

\section{Application}
\label{}

In this section, we delve into the potential applications of the energy of interval-valued fuzzy soft sets, as introduced in Section 3 of this paper. We aim to illustrate how this energy serve as a potent tool in devising decision-making algorithms. Understanding the interconnectedness among all factors within the system is crucial for effective decision-making. Factors with certain attributes may dominate the system, potentially leaving other segments uncovered. In scenarios where no single factor adequately covers a specific segment of the system, issues may arise. To illustrate this point, we reference specific examples outlined in the paper \cite{qin}, where the obtained results are compared with results based on algorithms proposed in \cite{ma} and \cite{yangint}. Within our study, we base decision-making in these examples on the defined energy of interval-valued fuzzy soft sets.

Let's first consider the example where a family needs to choose an apartment to buy from the work \cite{qin} by Qin et al. The authors proposed an algorithm to solve this problem and compared the results obtained with the results obtained by applying algorithms from the works \cite{yangint} and \cite{ma}.

\begin{example}\label{apartman}\cite{qin}
One family is planning to buy an apartment building for living. There are five alternative apartment candidates from five different property developers. This family hesitates about which to buy. We are able to evaluate the alternatives from four aspects:''reasonable price'',''excellent geographical location'', ''perfect facilities'', ''cozy environment''. The model of IVFSS is chosen to describe the customer’s feeling for the five candidates from four aspects. Hence, suppose that the universe $U$ represents the set of the five different alternative apartment candidates and $U = \{u_1, u_2, u_3, u_4, u_5\}.$ Then, $A$ represents the set of four parameters and 

$A = \{x_1, x_2, x_3, x_4\} = \{\text{reasonable price, excellent geographical location, perfect facilities, cozy environment}\}.$ 

Then IVFSS $\mathcal{F}_A$ on $U$ is described by the following table

$\begin{array}{r|cccc}
  \mathcal{F}_A & x_1 & x_2 & x_3 & x_4 \\ \hline
  u_1 & [0.3,0.5] & [0.6,0.7] & [0.2,0.4] & [0.4,0.5] \\
  u_2 & [0.3,0.4] & [0.4,0.5] & [0.6,0.7] & [0.1,0.3] \\
  u_3 & [0.5,0.6] & [1.0,1.0] & [0.2,0.3] & [0.2,0.4] \\
  u_4 & [0.5,0.7] & [0.0,0.1] & [0.7,0.8] & [0.6,0.7] \\
  u_5 & [0.3,0.6] & [0.3,0.4] & [0.4,0.7] & [0.2,0.3]
\end{array}.$
\end{example}

To solve the mentioned problem, we will formulate an algorithm for decision-making based on the energy of the interval-valued fuzzy soft set.

{\bf Step 1:} Input the interval-valued fuzzy soft set $\mathcal{F}_A$ over $U.$

{\bf Step 2:} Form interval-valued fuzzy soft sets $\mathcal{F}_{A_{i}}$, over $U\setminus{u_{i}}$ for each $u_{i}\in U$.

{\bf Step 3:} For each interval-valued fuzzy soft set from step 2, form their corresponding matrice of minimum and matrice of maximum values.

{\bf Step 4:} Determine the singular values for each obtained matrice of minimum and each obtained matrice of maximum values.

{\bf Step 5:} Determine the pessimistic energies $\mathbb{E}^{\min}_{\mathcal{F}_{A_i}}$ and the optimistic energies $\mathbb{E}^{\max}_{\mathcal{F}_{A_i}}$ for each interval-valued fuzzy soft sets $\mathcal{F}_{A_{i}}$ based on the obtained singular values.

{\bf Step 6:} Determine the energies $\mathbb{E}^*_{\mathcal{F}_{A_i}}$ for each interval-valued fuzzy soft sets $\mathcal{F}_{A_{i}}$ based on the obtained energies in step 5.

{\bf Step 7:} Determine the minimum energy among all energies of interval-valued fuzzy soft sets obtained in step 6 and interpret the result obtained.

Using Example \ref{apartman} we will illustrate how decisions are made based on the algorithm. The interval-valued fuzzy soft set $\mathcal{F}_A$ is represented by the table in Example \ref{apartman}, fulfilling the first step of the algorithm. In the second step, we need to form five interval-valued fuzzy soft sets, as described. We will illustrate how to find $\mathcal{F}_{A_1},$ and the remaining interval-valued fuzzy soft sets are determined analogously. We obtain the table for $\mathcal{F}_{A_1}$ by removing the first row of the table for $\mathcal{F}_A,$ resulting in

$\begin{array}{r|cccc}
  \mathcal{F}_{A_1} & x_1 & x_2 & x_3 & x_4 \\ \hline
  u_2 & [0.3,0.4] & [0.4,0.5] & [0.6,0.7] & [0.1,0.3] \\
  u_3 & [0.5,0.6] & [1.0,1.0] & [0.2,0.3] & [0.2,0.4] \\
  u_4 & [0.5,0.7] & [0.0,0.1] & [0.7,0.8] & [0.6,0.7] \\
  u_5 & [0.3,0.6] & [0.3,0.4] & [0.4,0.7] & [0.2,0.3]
\end{array}.$

The corresponding matrix of minimum values and the matrix of maximum values for $\mathcal{F}_{A_1}$ are
\[
\Psi_{\mathcal{F}_{A_1}}^{\min}=\begin{bmatrix}
                                  0.3 & 0.4 & 0.6 & 0.1 \\
                                  0.5 & 1.0 & 0.2 & 0.2 \\
                                  0.5 & 0.0 & 0.7 & 0.6 \\
                                  0.3 & 0.3 & 0.4 & 0.2 
                                \end{bmatrix}\quad\text{and}\quad\Psi_{\mathcal{F}_{A_1}}^{\max}=\begin{bmatrix}
                                  0.4 & 0.5 & 0.7 & 0.3 \\
                                  0.6 & 1.0 & 0.3 & 0.4 \\
                                  0.7 & 0.1 & 0.8 & 0.7 \\
                                  0.6 & 0.4 & 0.7 & 0.3 
                                \end{bmatrix},
\]
thus completing the third step.

The singular values of the matrix $\Psi_{\mathcal{F}_{A_1}}^{\min}$ are
\[
\sigma_1= 1.624965, \ \sigma_2=0.836748,\ \sigma_3=0.298884,\ \sigma_4= 0.002707.
\]
Hence, the pessimistic energy of the interval-valued fuzzy soft set $\mathcal{F}_{A_1}$ is the sum of the obtained singular values, i.e. $\mathbb{E}^{\min}_{\mathcal{F}_{A_1}}=2.763304.$

The singular values of the matrix $\Psi_{\mathcal{F}_{A_1}}^{\max}$ are
\[
\sigma_1= 2.151748, \ \sigma_2=0.768472,\ \sigma_3=0.306279,\ \sigma_4= 0.124988.
\]
Hence, the optimistic energy of the interval-valued fuzzy soft set $\mathcal{F}_{A_1}$ is the sum of the obtained singular values, i.e. $\mathbb{E}^{\max}_{\mathcal{F}_{A_1}}=3.351487.$

Now, the energy of the interval-valued fuzzy soft set $\mathcal{F}_{A_1}$ is
\[
\mathbb{E}^{*}_{\mathcal{F}_{A_1}}=3.0573955.
\]

Similarly, we determine the remaining energies and obtain:
\[
\mathbb{E}^{*}_{\mathcal{F}_{A_2}}=3.0316169,\ \mathbb{E}^{*}_{\mathcal{F}_{A_3}}=2.78006545,\ \mathbb{E}^{*}_{\mathcal{F}_{A_4}}=2.709927,\ \mathbb{E}^{*}_{\mathcal{F}_{A_5}}=3.157358.
\]

We need to draw a conclusion based on the obtained energies. Since the interval-valued fuzzy soft set $\mathcal{F}_{A_4}$ has the lowest energy, this means that the element $u_4$ contributes the most to the energy of the overall system, or it has the greatest influence on the systemic value, so apartment $u_4$ should be chosen. The obtained energy values can be linearly arranged as follows:
\[
\mathbb{E}^{*}_{\mathcal{F}_{A_4}}\leqslant\mathbb{E}^{*}_{\mathcal{F}_{A_3}}\leqslant\mathbb{E}^{*}_{\mathcal{F}_{A_2}}\leqslant\mathbb{E}^{*}_{\mathcal{F}_{A_1}}\leqslant\mathbb{E}^{*}_{\mathcal{F}_{A_5}}.
\]
Thus, by comparing all the energies, we obtain a linear order of apartments based on the priority of their selection:
\[
u_4\succ u_3\succ u_2\succ u_1\succ u_5.
\]

Let's compare the obtained solution with the solutions obtained by applying algorithms from the papers \cite{qin}, \cite{ma} and \cite{yangint}. The algorithm for solving decision-making problems based on interval-valued fuzzy soft sets in the paper \cite{yangint} is the score based decision making approach (SBDM), while in the paper \cite{ma} it is the decision making method considering the added objects (CAODM). Both algorithms give the same result:
\[
u_3\succ u_4\succ u_1\succ u_2\succ u_5,
\]
which means that apartment $u_3$ is chosen as the best option. However, Qin et al. explained in their paper \cite{qin} that apartment $u_3$ is not the best choice because it has the best geographical location, but the other three parameters are significantly worse compared to apartment $u_4.$ In fact, apartment $u_4$ only has a poor geographical location, while all other characteristics are at the highest level compared to the other apartments. Therefore, it seems that $u_4$ is a better choice than $u_3.$ In this example, where the elements $u_3$ i $u_4$ have extreme values of some parameters, the SBDM and CAODM algorithms ignore excellent candidates. In the paper \cite{qin} an algorithm based on means of the contrast table (MCTDM) is proposed to solve this problem, and the following order of apartments is obtained:
\[
u_4\succ u_3\succ u_5\succ u_1\succ u_2.
\]

Therefore, with our proposed method based on the energy, apartment $u_4$ is chosen, just like with the MCTDM method, which means that our method takes into account outliers or extreme values of some specific parameters.

Now we will consider another example from the paper \cite{qin} related to real-life application, also compared with SBDM (\cite{yangint}) and CAODM (\cite{ma}). 

\begin{example}
A person has seven days off for his annual holiday. He desires to spend and enjoy his holiday at one scenic spot. He goes to visit the web site of www.weather.com.cn (accessed on 1 September 2021), which displays the weather forecast for sixteen destination scenic spots. The data of weather forecast are described from four aspects such as ''temperature'', ''relative humidity'', ''air quality index'', ''wind speed''. Qin et al. find the maximum and minimum values for every parameter as the upper limits and lower limits of the interval and then apply the model of IVFSS to describe this Scenic Spots Weather Condition Evaluation Systems. There are sixteen scenic spots in China as the candidates, ie. 
\begin{align*}
  U &= \{u_1,u_2,\ldots,u_{16}\} \\ 
    &= \{\text{Forbidden City, The Bund, Bangchui Island, West Lake, Five Avenue, Ciqikou, Confucius Temple,} \\
    &\ \quad \text{Yellow Crane Tower, Mount Tai, Jiuzhai Valley, Zhangjiajie, Gulangyu Islet,}  \\
    &\ \quad \text{The ancient City of Ping Yao, Terra Cotta Warriors, Mogao Grottoes, Erhai Lake}\}
\end{align*} 
and the set of parameters is
\begin{align*}
  A &= \{x_1,x_2,x_3,x_4\}=\{\text{''temperature'', ''relative humidity'', ''air quality index'', ''wind speed''}\}.
\end{align*}

Qin et al. normalized the original data into IVFSS by transforming minimum value and maximum value into subintervals of $[0,1],$ which is normalized as lower and upper degree of membership. Then the table of obtained IVFSS is given below

$\begin{array}{r|cccc}
   \mathcal{F}_A & x_1 & x_2 & x_3 & x_4 \\ \hline
      u_1 & [0.13,0.52] & [0.60,0.98] & [0.27,0.95] & [0.50,1.00] \\
      u_2 & [0.35,0.71] & [0.06,0.49] & [0.62,0.89] & [0.25,1.00] \\
      u_3 & [0.23,0.52] & [0.36,0.64] & [0.25,0.81] & [0.50,0.75] \\
      u_4 & [0.39,0.74] & [0.06,0.37] & [0.46,0.74] & [0.25,0.75] \\
      u_5 & [0.19,0.55] & [0.50,1.00] & [0.04,0.76] & [0.50,0.75] \\
      u_6 & [0.45,0.81] & [0.06,0.20] & [0.73,0.85] & [0.50,1.00] \\
      u_7 & [0.32,0.71] & [0.00,0.51] & [0.57,0.73] & [0.50,0.75] \\
      u_8 & [0.39,0.77] & [0.01,0.10] & [0.57,0.96] & [0.50,0.75] \\
      u_9 & [0.03,0.42] & [0.33,0.89] & [0.06,0.76] & [0.00,0.75] \\
   u_{10} & [0.06,0.42] & [0.20,0.48] & [0.37,0.67] & [0.75,1.00] \\
   u_{11} & [0.42,0.81] & [0.01,0.16] & [0.91,1.00] & [0.50,0.75] \\
   u_{12} & [0.65,1.00] & [0.32,0.44] & [0.66,0.91] & [0.00,0.50] \\
   u_{13} & [0.13,0.58] & [0.19,0.90] & [0.04,0.55] & [0.25,1.00] \\
   u_{14} & [0.19,0.68] & [0.15,0.34] & [0.00,0.36] & [0.50,0.75] \\
   u_{15} & [0.00,0.48] & [0.32,0.91] & [0.29,0.78] & [0.50,0.75] \\
   u_{16} & [0.42,0.87] & [0.20,0.78] & [0.46,0.92] & [0.50,1.00]
 \end{array}
$
\end{example}

We can apply the described decision-making algorithm using the energy of the interval-valued fuzzy soft set. Similar to the previous example, we calculate the singular values of the matrix of minimum values and the matrix of maximum values, as well as the pessimistic and optimistic energies of the obtained interval-valued fuzzy soft sets, resulting in the following energies of IVFSS:
\begin{align*}
  \mathbb{E}^{*}_{\mathcal{F}_{A_1}} &= 6.2798295, \mathbb{E}^{*}_{\mathcal{F}_{A_2}} = 6.383722, 
  \mathbb{E}^{*}_{\mathcal{F}_{A_3}} = 6.4203115, 
  \mathbb{E}^{*}_{\mathcal{F}_{A_4}} = 6.4340589, \\
  \mathbb{E}^{*}_{\mathcal{F}_{A_5}} &= 6.3269075, 
  \mathbb{E}^{*}_{\mathcal{F}_{A_6}} = 6.3179751, 
  \mathbb{E}^{*}_{\mathcal{F}_{A_7}} = 6.4112738, 
  \mathbb{E}^{*}_{\mathcal{F}_{A_8}} = 6.345155, \\
  \mathbb{E}^{*}_{\mathcal{F}_{A_9}} &= 6.4089085, 
  \mathbb{E}^{*}_{\mathcal{F}_{A_{10}}} = 6.336624, 
  \mathbb{E}^{*}_{\mathcal{F}_{A_{11}}} = 6.27166, 
  \mathbb{E}^{*}_{\mathcal{F}_{A_{12}}} = 6.1741684, \\
  \mathbb{E}^{*}_{\mathcal{F}_{A_{13}}} &= 6.3893555, 
  \mathbb{E}^{*}_{\mathcal{F}_{A_{14}}} = 6.383308, 
  \mathbb{E}^{*}_{\mathcal{F}_{A_{15}}} = 6.372367, 
  \mathbb{E}^{*}_{\mathcal{F}_{A_{16}}} = 6.35310915.
\end{align*}

The obtained values of energies can be linearly arranged as follows:
\begin{align*}
  \mathbb{E}^{*}_{\mathcal{F}_{A_{12}}} &\leqslant\mathbb{E}^{*}_{\mathcal{F}_{A_{11}}}\leqslant\mathbb{E}^{*}_{\mathcal{F}_{A_1}}\leqslant\mathbb{E}^{*}_{\mathcal{F}_{A_6}}\leqslant\mathbb{E}^{*}_{\mathcal{F}_{A_5}}
  \leqslant\mathbb{E}^{*}_{\mathcal{F}_{A_{10}}}\leqslant\mathbb{E}^{*}_{\mathcal{F}_{A_8}}\leqslant\mathbb{E}^{*}_{\mathcal{F}_{A_{16}}}\leqslant\mathbb{E}^{*}_{\mathcal{F}_{A_{15}}} \\
   &\leqslant\mathbb{E}^{*}_{\mathcal{F}_{A_{14}}}\leqslant\mathbb{E}^{*}_{\mathcal{F}_{A_2}}\leqslant\mathbb{E}^{*}_{\mathcal{F}_{A_{13}}}\leqslant\mathbb{E}^{*}_{\mathcal{F}_{A_9}}
   \leqslant\mathbb{E}^{*}_{\mathcal{F}_{A_7}}\leqslant\mathbb{E}^{*}_{\mathcal{F}_{A_3}}\leqslant\mathbb{E}^{*}_{\mathcal{F}_{A_4}}.
\end{align*}

Thus, we can conclude that we obtain the following order of the scenic spots in China:
\[
u_{12}\succ u_{11}\succ u_{1}\succ u_{6}\succ u_{5}\succ u_{10}\succ u_{8}\succ u_{16}\succ u_{15}\succ u_{14}\succ u_{2}\succ u_{13}\succ u_{9}\succ u_{7}\succ u_{3}\succ u_{4}.
\]

If we compare the obtained results with the results obtained by applying the aforementioned algorithms, we can conclude the following. By applying the SBDM and CAODM algorithms, the following ranking is obtained:
\[
u_{16}\succ u_{1}\succ u_{6}\succ u_{11}\succ u_{12}\succ u_{2}\succ u_{5}\succ u_{7}\succ u_{3}\succ u_{8}\succ u_{15}\succ u_{10}\succ u_{4}\succ u_{13}\succ u_{9}\succ u_{14},
\]
while applying the MCTDM algorithm yields:
\[
u_{16}\succ u_{6}\succ u_{1}\succ u_{12}\succ u_{11}\succ u_{2}\succ u_{8}\succ u_{7}\succ u_{3}= u_{5}\succ u_{15}= u_{10}\succ u_{13}\succ u_{4}\succ u_{14}\succ u_{9}.
\]
We can notice that all methods provide similar solutions, with our algorithm giving preference to places like Gulangyu Islet and Zhangjiajie over Erhai Lake. Although the MCTDM method is an improvement over the SBDM and CAODM methods, it also has certain drawbacks. Namely, the question of the uniqueness of the decision-making order arises, where in this example, we can observe that the MCTDM method does not provide a unique solution. Our method based on energy improves in terms of providing a unique solution, but it has a slightly different linear order compared to the MCTDM method from the paper \cite{qin}.

Finally, the following table lists the basic characteristics of the observed methods obtained based on the examples provided in this paper.

\begin{tabular}{c|c|c|c}
  Procedure & Ranking      & Considering the  & Unique \\ 
                   & methodology   & extreme data     & solution             \\ \hline
  
  \cite{yangint} & Scores based & No & Yes \\
  & decision-making approach &  & \\ \hline 
  
  \cite{ma} & Decision-making method & No & Yes \\
   & considering the added objects &  &  \\ \hline
   
  \cite{qin} & Means of the & Yes & No \\
   & contrast table &  &  \\ \hline
  Algorithm based on & Scores based on  & Yes & Yes \\
  the energy of IVFSS & comparison of energies &  & 
\end{tabular}\\

After analyzing the obtained results, we can conclude that our method based on the pessimistic energy, the optimistic energy and the energy of interval-valued fuzzy soft sets does not fully correspond to the proposed solutions in the observed papers. The level of confidence in the decision correlates with the minimum or maximum energy of the interval-valued fuzzy soft set, and the main challenge of methods relying on the energy of interval-valued fuzzy soft sets is the difficulty in determining the exact lower and upper bounds of this energy.

\section{Summary and Conclusion}
\label{}
Interval-valued fuzzy soft sets are obtained by combining interval-valued fuzzy sets and soft sets. They find wide applications in numerous areas, including decision-making. In this paper, we manage to integrate characteristics from graph theory into the theory of IVFSS, paving the way for new research possibilities. The main focus is on defining numerical characteristics of interval-valued fuzzy soft sets - the pessimistic energy, the optimistic energy and the energy of IVFSS. The energy is obtained as the arithmetic mean of the pessimistic and optimistic energies of the interval-valued fuzzy soft set. It contributes to the improvement of decision-making algorithms, and in this paper, a parallel is drawn between the decision-making algorithm based on the energy of IVFSS and algorithms defined in the papers \cite{ma}, \cite{qin} and \cite{yangint}. There are numerous questions that can be considered in future research, such as the boundaries of the introduced energies. Additionally, one of the future research directions is the introduction of IVFSS energy when the weights of given parameters are known, and the construction of a decision-making algorithm based on this introduced energy that can be compared with the algorithm in the paper \cite{peng}.


\end{document}